\documentclass[10pt,journal,compsoc]{IEEEtran}
\ifCLASSOPTIONcompsoc
  \usepackage[nocompress]{cite}
\else
  \usepackage{cite}
\fi

\usepackage{outline}
\usepackage{pmgraph}
\usepackage[normalem]{ulem}
\usepackage[utf8]{inputenc}
\usepackage{amssymb}
\usepackage{hyperref}
\usepackage{amsmath}
\usepackage{graphicx}
\usepackage{times}
\usepackage{xcolor}
\usepackage{xspace}
\usepackage[colorinlistoftodos]{todonotes} % for \todo
\usepackage{cite}
\usepackage{bm} % bold
\usepackage{url}

\usepackage{epstopdf}
\AppendGraphicsExtensions{.tiff}
\graphicspath{{fig/}} % Directory in which figures are stored

\usepackage{epsfig}
\usepackage{tikz}
\usetikzlibrary{spy}
\usepackage{algpseudocode}
\usepackage{algorithm}
\usepackage{mathrsfs}

\usepackage{nicefrac}       % compact symbols for 1/2, etc.
\usepackage{booktabs}       % professional-quality tables

\usepackage{amsmath,graphicx,caption,mathtools,amssymb,amsthm}
\usepackage{graphicx}
\usepackage{multirow}
\usepackage{subcaption}
\usepackage{accents}
\newtheorem{proposition}{\bf{Proposition}}

\long\def\comment#1{} % comment out text

\newcommand{\xmath}[1] {\ensuremath{#1}\xspace}
\newcommand{\blmath}[1] {\xmath{\bm{#1}}}

\newcommand{\0}{\blmath{0}}

\newcommand{\Ab}{{\blmath A}}

\newcommand{\Fb}{{\blmath F}}

\newcommand{\Ib}{{\blmath I}}

\newcommand{\Mb}{{\blmath M}}

\newcommand{\Pb}{{\blmath P}}

\newcommand{\Xb}{{\blmath X}}

\newcommand{\bb}{{\blmath b}}

\newcommand{\kb}{{\blmath k}}

\newcommand{\rb}{{\blmath r}}

\newcommand{\ub}{{\blmath u}}

\newcommand{\xb}{{\blmath x}}

\newcommand{\zb}{{\blmath z}}

\newcommand{\Hc}{\mathcal{H}}
\newcommand{\Gc}{\mathcal{G}}

\newcommand{\Bc}{\mathcal{B}}

\newcommand{\Qc}{\mathcal{Q}}

\newcommand{\Xc}{\mathcal{X}}

\newcommand{\Rd}{{\mathbb R}}
\newcommand{\Cd}{{\mathbb C}}

\newcommand{\beq}{\begin{equation}}
\newcommand{\eeq}{\end{equation}}
\newcommand{\beqa}{\begin{eqnarray}}
\newcommand{\eeqa}{\end{eqnarray}}

\begin{document}

\title{DeepPhaseCut: Deep Relaxation in Phase for Unsupervised Fourier Phase Retrieval}

\author{Eunju~Cha,~Chanseok~Lee,~Mooseok~Jang,~and~Jong~Chul~Ye,~\IEEEmembership{Fellow,~IEEE}% <-this % stops a space
\IEEEcompsocitemizethanks{\IEEEcompsocthanksitem E. Cha, C. Lee, M. Jang, and J.C. Ye are with the Department of Bio and Brain Engineering, Korea Advanced Institute of Science and Technology (KAIST), Daejeon 34141, Republic of Korea.\protect\\
E-mail: \{eunju.cha, cslee, mooseok, jong.ye\}@kaist.ac.kr}% <-this % stops an unwanted space
\thanks{Manuscript received;}}

% The paper headers
\markboth{Journal of \LaTeX\ Class Files}%
{Cha \MakeLowercase{\textit{et al.}}: DeepPhaseCut: Deep Relaxation in Phase for Unsupervised Phase Retrieval}

% for Computer Society papers, we must declare the abstract and index terms
% PRIOR to the title within the \IEEEtitleabstractindextext IEEEtran
% command as these need to go into the title area created by \maketitle.
% As a general rule, do not put math, special symbols or citations
% in the abstract or keywords.

\IEEEtitleabstractindextext{%
\begin{abstract}
Fourier phase retrieval is a classical problem of restoring a signal only from the measured magnitude of its Fourier transform. Although Fienup-type algorithms, which use prior knowledge in both spatial and Fourier domains, have been widely used in practice, they can often stall in local minima.
  Modern methods such as PhaseLift and PhaseCut may offer performance guarantees with the help of convex relaxation. However, these algorithms are usually computationally intensive for practical use.
To address this problem, we propose a novel, unsupervised, feed-forward neural network for Fourier phase retrieval which enables immediate high quality reconstruction.
% by exploiting optimal transport driven cycle-consistent generative adversarial network (cycleGAN) architecture.
Unlike the existing deep learning approaches that use a neural network
as a regularization term or an end-to-end blackbox  model for supervised training,
our algorithm is a feed-forward neural network implementation of PhaseCut algorithm in an unsupervised learning framework.
Specifically, our network is composed of two generators: one for the phase estimation using  PhaseCut loss,
followed by another generator for image reconstruction, all of which are trained simultaneously using a cycleGAN framework without matched data.
The link to the classical Fienup-type algorithms and the recent symmetry-breaking learning approach is also revealed.
%
% derived using cycleGAN architecture, composed of a single pair of generator and discriminators.
%%
%%
%%our cycleGAN consists of two pairs of generator and discriminator to exploit information both in spatial and Fourier domains. One generator was trained in the Fourier domain, while another was trained in the spatial domain.
%%The first generator estimated the phase from the given magnitude in the Fourier domain. The second generator then refined the intermediate results, formed by the given magnitude and the estimated phase, in the spatial domain.
%Furthermore, to make the generator phase agnostic, we introduce a new distributional matching framework by using multiple realization of phase noises. 
Extensive experiments demonstrate that the proposed method outperforms  all existing approaches in Fourier phase retrieval problems.
\end{abstract}

% Note that keywords are not normally used for peerreview papers.
}

% make the title area
\maketitle

\IEEEdisplaynontitleabstractindextext

\IEEEraisesectionheading{\section{Introduction}\label{Intro}}

\IEEEPARstart{F}{ourier} phase retrieval is a classical problem to recover a signal when only the magnitudes of its Fourier transform are available. 
%In a discrete form, the phase retrieval problem can be stated as to find the signal $\xb=\begin{bmatrix} x_1 & \cdots & x_N\end{bmatrix}^\top\in \Cd^N$ that satisfies
%\begin{align}
%b_i = |\langle \ab_i, \xb\rangle|,&~\quad i=1,\cdots, M
%\end{align}
%where $\bb=\begin{bmatrix} b_1 & \cdots & b_M\end{bmatrix}^\top$ is the Fourier magnitude measurement,
%and  $\ab_i$ denotes the $i$-th column of the DFT matrix
Specifically, for a given  signal $x(\rb),\rb\in \Rd^d$ with $d=2,3$, the measurement ${y}(\kb),\kb\in \Rd^d$ is given by 
\begin{equation}\label{eq:pr}
{y}(\kb)=\left(\mathcal{F}x \right)(\kb)  :=\Bigg\vert \int e^{-\iota\kb\cdot \rb}x(\rb)d\rb \Bigg\vert ,
\end{equation}
where %$\Hc: \xb \mapsto \hat{\yb}$ is a measurement operation and 
$\mathcal{F}$ is spatial Fourier transform  and $\iota=\sqrt{-1}$, and the goal of the Fourier phase retrieval is to estimate
the unknown signal $x$ from the magnitude measurement $y$.
%$\rb \in \Rd^2$ and  $\kb\in\Rd^2$ denote spatial coordinate and  spatial frequency, respectively. 
This problem occurs in many imaging applications such as X-ray crystallography \cite{millane1990phase, harrison1993phase}, microscopy \cite{zheng2013wide}, astronomical imaging \cite{fienup1987phase}, diffraction imaging \cite{bunk2007diffractive}, etc., where it is infeasible or highly difficult to obtain the phase information of the signal during the acquisition. 

%The uniqueness condition for the phase-retrieval problem from the Fourier magnitude
%measurement have been studied \cite{bates1984uniqueness,miao1998phase,rodriguez2013oversampling}.
While the  uniqueness condition for the Fourier phase retrieval can be obtained from appropriate oversampling \cite{hayes1982reconstruction,bruck1979ambiguity,bates1982fourier,shechtman2015phase}, there still exists ambiguity in the solutions due to the symmetry, such as translation, 
flipping, etc. Therefore, for a given measurement, there exists infinitely many solutions, so
phase retrieval is generally considered a very challenging ill-posed inverse problem. 
%For example, t
% which is a typical example of ill-posedness.

Classical algorithms for the phase retrieval problems are based on alternating projection  approaches \cite{gerchberg1972practical,fienup1978reconstruction,fienup1982phase}.
%For example, Gerchberg-Saxton algorithm \cite{gerchberg1972practical} alternates between orthogonal projections and normalization of the magnitudes
%to match the observations.
%Hybrid input-output (HIO) algorithm or Fienup algorithm \cite{fienup1982phase} extends Gerchberg-Saxton algorithm
%% to restore the phase of signal 
% by successively imposing the known magnitude of Fourier transform  and constraints in the spatial domain such as finite support and nonnegativity of signal at each iteration.
%Even though the convergence of the HIO algorithm to the global solution is not guaranteed, this algorithm is the most commonly utilized thanks to its low computational complexity and the tendency to avoid local minima in the absence of noise. 
%These alternating projection-based methods have been developed by grafting an oversampling of the magnitude of Fourier transform $\hat{y}(\kb)$ to ensure the convergence of the algorithm \cite{bates1984uniqueness,miao1998phase,rodriguez2013oversampling}. 
The sparsity of the signal can also be exploited as a prior knowledge to solve the phase retrieval problem \cite{jaganathan2013sparse, eldar2014sparse, pauwels2017fienup}.
While empirical performance is often good,
these algorithms often suffer from local minima.

To address the local minima problem, convex relaxation approaches such as  PhaseLift \cite{candes2013phaselift} and PhaseCut \cite{waldspurger2015phase} were introduced.
For example,
PhaseLift has received significant attention due to its link to compressed sensing, by lifting the original problem to a higher dimensional rank-one matrix recovery problem
over positive semidefinite Hermitian matrices satisfying some linear conditions.
On the other hand,  Waldspurger et al \cite{waldspurger2015phase} 
showed that the phase retrieval problem can
be reformulated  as an extension of the MAXCUT
combinatorial graph partitioning problem \cite{ding2001min} over the unit complex torus, leading to a
semidefinite relaxations of MAXCUT for the phase retrieval problem.
However, the computational complexity of these convex relaxation approaches are still often too expensive for practical applications.

Recently, various forms of deep learning approaches have been investigated as a fast and effective approach for phase retrieval \cite{metzler2018prdeep,hand2018phase,sinha2017lensless,icsil2019deep,uelwer2019phase, hyder2020solving,metzler2020deep,goy2018low}.
%For example,
% a pre-trained network is  used as a regularization of the optimization problem \cite{metzler2018prdeep,hand2018phase,hyder2020solving},
%  or an untrained neural network is used as a structural prior to restrict the feasible
%set \cite{?}.
%%Supervised training is usually used using matched data set.
%%$\{ \hat{\yb}^{(m)}, \xb^{(m)}\}_{m=1}^M$, where $M$ is the number of training data. 
%Specifically, in \cite{metzler2018prdeep}, the authors  leveraged the regularization-by-denoising (RED) framework \cite{?} and a convolutional neural network denoiser,
%whereas in \cite{icsil2019deep}, and iterative hybrid  of neural networks and the model-based HIO method was proposed. The authors in \cite{uelwer2019phase} proposed to employ conditional generative adversarial network (cGAN) to restore the phase information of the signal.
All of these algorithms have demonstrated high performance compared to the traditional alternating projection-based methods.
Nonetheless, iterative schemes using a pretrained network as a regularization term \cite{metzler2018prdeep,hand2018phase} leads to increased computational complexity,
whereas feed-forward methods usually require extensive amount of training data and are often sensitive to the inherent ambiguity of the solutions
\cite{metzler2020deep,goy2018low,uelwer2019phase,icsil2019deep}.
%Moreover, the link to the modern convex relaxation approaches \cite{?} is not clear in these approaches.

To address these issues,  this paper provides a novel unsupervised feed-forward neural network approach, dubbed DeepPhaseCut, which provides
accurate reconstruction in a real-time manner.
Inspired by the classical alternating projection methods \cite{gerchberg1972practical,fienup1978reconstruction,fienup1982phase} and the modern PhaseCut approach \cite{waldspurger2015phase},
DeepPhaseCut is composed of two neural networks: one for phase estimation
and the other for image reconstruction, which are trained simultaneously in an end-to-end manner. Specifically,  we propose 
a novel PhaseCut loss function for phase estimation in the first neural network.
Using the phase estimation from the first network, the second neural network  then refines the image reconstruction.
%, which can be considered as a neural network implementation of similar  reconstruction step in  PhaseCut or Fienup algorithm.
%Then, the two generators are trained in end-to-end manner.
%neural network.
Inspired by  symmetry-breaking end-to-end learning scheme for phase retrieval  \cite{manekar2020},
another important contribution of this paper is the use of the adversarial loss and cycle consistency loss for the training of the overall networks, which leads to an architecture similar to cycleGAN \cite{zhu2017unpaired}. 
Accordingly, the new losses enable the unsupervised training of the neural network without matched label data. % using a cycleGAN framework.
Using extensive experimental results, we confirm that the proposed method can restore finer details of the signal  quickly
without any matched reference data. %complexity compared to the existing approaches.

This paper is structured as follows. In Section~\ref{sec:related}, we  provide an extensive review of uniqueness results and the existing algorithms for the Fourier phase retrieval problem.
We then derive the proposed method as an extension of the PhaseCut approach in Section~\ref{sec:theory}.
Implementation details are provided in Section~\ref{sec:method}. Experimental results are provided in Section~\ref{sec:experiments}, which is followed by discussion and conclusion in Section~\ref{sec:discussion} and Section~\ref{sec:conclusion}, respectively.

\section{Related Works}
\label{sec:related}

\subsection{Uniqueness Condition}

The uniqueness condition for the Fourier phase retrieval problem is usually studied
in a discrete setup. Specifically, in  the $2$-D phase retrieval problem,
which is of main interest in this paper,
the measurement is given by
\begin{align}\label{eq:b}
b[k_1,k_2] &= | \widehat x[k_1,k_2]|,\quad 0\leq k_i\leq M_i-1,
\end{align}
where $\widehat x[k_1,k_2]$ is  {\em oversampled} discrete Fourier transform (DFT):
%\begin{align}
%[\kb] &=\sum_{\nb=\0}^{\Nb-1}  x[\nb] \exp\left(-\iota 2\pi \frac{\kb^\top \nb}{\Mb}\right)
%\end{align}
\begin{align}\label{eq:DFT}
\widehat x[k_1,k_2] &=\frac{1}{\sqrt{M_1M_2}}\sum_{n_1,n_2=0}^{N_1,N_2}  x[n_1,n_2] e^{-\iota 2\pi \left(\frac{k_1 n_1}{M_1}+\frac{k_2 n_2}{M_2} \right)},
\end{align}
where $M_i \geq N_i$ for $i=1,2$. 
In fact, \eqref{eq:DFT} is a standard normalized DFT for the finite supported signal, i.e $x[n_1,n_2]=0$ when
$n_i> N_i-1$ or $n_i<0$ for any $i$.

It was shown in \cite{hayes1982reconstruction,bruck1979ambiguity,bates1982fourier,shechtman2015phase} that the  2-dimensional signal  $x[n_1,n_2]$ is uniquely specified  by the  magnitude
of the oversampled DFT almost surely with $M_i \geq N_i$ for $i=1,2$. Furthermore, the oversampling factor of
$M_i \geq 2N_i - 1$ for $i=1,2$ is sufficient to
guarantee the uniqueness.
However, care should be taken since the uniqueness in the phase retrieval is up to the signals in the equivalent class.
Specifically, an equivalent class to the signal $x[n_1,n_2]$ in the Fourier phase retrieval is defined as
\begin{align}
y[n_1,n_2]&\sim x[n_1,n_2],
\end{align}
if
\begin{align}\label{eq:equiv}
y[n_1,n_2]=e^{\iota\theta} x[s_1\pm n_1,s_2\pm n_2],
\end{align}
for some integer-valued shift vector $[s_1,s_2]$ and the global phase $\theta\in \Rd$. In other words, the equivalence
class is defined to be
the set of all signals which may be derived from $x[n_1,n_2]$ by
shifting,  time-reversal, or change in the global phase \cite{hayes1982reconstruction,bruck1979ambiguity,bates1982fourier,shechtman2015phase}.

\subsection{Classical  Alternating Projection Approaches}

Alternating projection-based methods such as Gerchberg-Saxton \cite{gerchberg1972practical} and hybrid input-output (HIO) algorithm by Fienup \cite{fienup1982phase}
imposes both Fourier domain and spatial domain constraint at each iteration.

To explain these algorithm more concisely, 
the measurement model with \eqref{eq:b} and \eqref{eq:DFT} is simply 
represented by a matrix form:
\begin{align}
\label{eq:fw}
\bb = T(\xb):= |\Ab\xb|,&%~\quad m=1,\cdots, M
\end{align}
where $|\cdot|$ denotes the element-wise absolute value operation, $T(\cdot)$ refers to the forward operator from the image to the Fourier magnitude measurement, and
\begin{eqnarray*}
\bb=\begin{bmatrix} b[1] \\ \vdots\\ b[M]\end{bmatrix},&& \xb=\begin{bmatrix} x[1] \\ \vdots \\ x[N]\end{bmatrix},
\end{eqnarray*}
are the vectorized version of $b[k_1,k_2]$
and $x[n_1,n_2]$, respectively, with $N=N_1N_2$ and $M=M_1M_2$.
Furthermore,
 $\Ab\in \Cd^{M\times N}$ is a $M\times N$ submatrix of $M\times M$ DFT matrix
 that corresponds to the non-zero support of $x[n_1,n_2]$, which
%Furthermore, the oversampled DFT matrix $\Ab$ 
can be represented by
\begin{align}\label{eq:A}
\Ab  =\Fb_M \Pb_{N},
\end{align}
where $\Fb_M \in \Cd^{M\times M}$ is the  $M\times M$-DFT matrix,
and $\Pb_N\in \Cd^{M\times N}$ denotes the zero-padding matrix converting $\xb\in \Cd^N$
to a zero-padded $M$-dimensional matrix $\Pb_N\xb\in \Cd^M$.

%For example, Gerchberg-Saxton algorithm \cite{gerchberg1972practical} alternates between orthogonal projections and normalization of the magnitudes
%to match the observations.
%Hybrid input-output (HIO) algorithm or Fienup algorithm \cite{fienup1982phase} extends Gerchberg-Saxton algorithm
%% to restore the phase of signal 
% by successively imposing the known magnitude of Fourier transform  and constraints in the spatial domain such as finite support and nonnegativity of signal at each iteration.
%Even though the convergence of the HIO algorithm to the global solution is not guaranteed, this algorithm is the most commonly utilized thanks to its low computational complexity and the tendency to avoid local minima in the absence of noise. 

With this notation,  the alternating projection algorithms can be represented by the following update:
% let $\nb=[n_1,n_1],\kb=[k_1,k_2]$ and
%let  $\widehat x_i[\kb]$ denote the Fourier transform of $x_i[\nb]$.
%Then, % by denoting $\nb=[n_1,n_2]$, %
%the signal update at the $(i+1)$-th iteration  is computed by
\begin{align}
\ub^{(i)} &:=\Gc\left(\widehat \xb^{(i)}\right)  \label{eq:s1} \\
\zb^{(i)} &= \Fb_M^\top \left(\bb\odot \ub^{(i)}\right) \label{eq:s2} \\
\xb^{(i+1)}&:= \Hc\left( \zb^{(i)}\right) \label{eq:s3}
\end{align}
where   $\widehat \xb^{(i)}=\Ab\xb^{(i)}$ is the $i$-th estimate of  Fourier spectrum, $\ub^{(i)}$ refers the the $i$-th estimate of the phase components on the unit complex modulus, 
$\bb$ denotes the Fourier magnitude measurement, $\odot$ is the element-wise product (Hadamard product),
and $\zb^{(i)}$ denotes the $i$-th iteration intermediate image reconstruction. Accordingly,
$\Gc$ and $\Hc$ refer to the phase update  and image refinement functions, respectively.
%
% $G_X$ is an image update function,  and $z^{(i)}[n]$ is the intermediate image
%whose Fourier spectrum is updated by
%\begin{align}
%\widehat  z^{(i)}[m]&:=b[m] G_U\left(\widehat x^{(i)}[m]\right) 
%\end{align}
%where $b[m]=|\widehat x[m]|$ is the  Fourier magnitude measurement 
%and $G_U$ is a phase update function.

The phase update function for Gerchberg-Saxton \cite{gerchberg1972practical} and HIO algorithm by Fienup \cite{fienup1982phase}
is identically given by
\begin{align}
\Gc(\widehat \xb^{(i)})[n]  := \frac{\widehat x^{(i)}[m]}{|\widehat x^{(i)}[m]|}.
\end{align}
On the other hand, the image update functions are different.
In particular,   Gerchberg-Saxton \cite{gerchberg1972practical} assumes the known magnitude in the image domain so that
the image update
is given by
\begin{align}
\Hc(\zb^{(i)})[n]  :=  |x[n]| \frac{ z^{(i)}[m]}{|z^{(i)}[m]|},
\end{align}
with $|x[n]|$ being the known magnitude in the image domain, whereas
the HIO algorithm \cite{fienup1982phase} imposes the image domain constraint:
\begin{align}
\Hc(\zb^{(i)})[n]   := \begin{cases}
z^{(i)}[n], & n \notin \Gamma, \\
x^{(i)}[n] - \beta z^{(i)}[n], & n \in \Gamma,
\end{cases}
\label{eq:HIO}
\end{align}
for some parameter $\beta>0$, 
where $\Gamma$ is a set of indices for which the intermediate update $z^{(i)}[n]$ violates the spatial constraints. 

Since spatial constraints such as sparsity, positivity, and support constraint can be easily incorporated,
HIO and its variants \cite{luke2004relaxed,chen2007application,rodriguez2013oversampling} usually offer a suitable reconstruction with little computing effort. Unfortunately, the convergence of the HIO algorithm to the global solution is not guaranteed, and the  practical convergence  largely depends on the initialization and the presence of noise.

\subsection{Convex Relaxation Approaches}

The modern phase retrieval approaches called PhaseLift \cite{candes2013phaselift} and PhaseCut \cite{waldspurger2015phase}  were developed
to address the local minima problems in the alternating minimization approaches via convex relaxation.
In particular, the two algorithms can be considered as dual to each other in that PhaseLift relies on the semidefinite relaxation of the signal, whereas PhaseCut exploits the semidefinite relaxation of phase.

More specifically, PhaseLift is based on the following observation:
\begin{align}
| \Ab_m\xb|^2=b_m^2 &\Longleftrightarrow \mathrm{Tr}(\Ab_m^\top \Ab_m \Xb) =b_m^2,
\end{align}
where  $\Ab_m$ denotes the $m$-th row of $\Ab$ and
$\mathrm{Tr}(\cdot)$ denotes the matrix trace, and $\Xb = \xb\xb^\top \in \Cd^{N\times N}$.
Note that $\Xb$ is a rank-1 matrix.
This observation leads to the following optimization problem:
\begin{align*}
\mbox{minimize} \quad & \mathrm{Rank}(\Xb), \\
\mbox{subject to} \quad& \mathrm{Tr}(\Ab_m^\top \Ab_m \Xb)= b_m^2,~m=1,\cdots, M\\
& \Xb \succeq \0,
\end{align*}
where $\Xb \succeq \0$ denotes the positive semidefinite matrices.
Since the rank minimization problem is nonconvex, PhaseLift resorts to the semidefinite relaxation
by replacing it with the minimization of $\mathrm{Tr}(\Xb)$. This semidefinite relaxation has been shown to recover the true signal $\xb$ exactly under
some oversampling factors \cite{candes2013phaselift}.

On the other hand, PhaseCut \cite{waldspurger2015phase}  starts with the following minimization problem:
\begin{align}
\min\limits_{\ub,\xb} &~ \|\Ab\xb-\mathrm{diag}(\bb)\ub\|^2 \label{eq:phasecut}\\
\mbox{subject to} &~|u_m|=1,~m=1,\cdots, M 
\end{align}
where we optimize over both the phase $\ub \in \Cd^M$ and the signal $\xb\in \Cd^N$.
The problem can be jointly minimized with respect to $\ub$ and $\xb$. More specifically,
for a given $\ub$, the signal estimate is given by
\begin{align}
\xb = \Ab^\dag \mathrm{diag}(\bb)\ub,
\end{align}
where $\Ab^\dag$ is the pseudo-inverse of $\Ab$. By plugging this in \eqref{eq:phasecut}, we have
\begin{align}
\min_{\ub\in \Cd^M} &~\ub^\top \Mb \ub \label{eq:lM} \\
\mbox{subject to} &~|u_i|=1,~i=1,\cdots, m  \notag
\end{align}
where
\begin{align}\label{eq:M}
\Mb:=  \mathrm{diag}(\bb)(\Ib -\Ab\Ab^\dag )  \mathrm{diag}(\bb).
\end{align}
It turns out that  the resulting problem is nonconvex with respect to the phase $\ub$, so  a convex relaxation  in the form of
 semidefinite program (SDP)  was also derived  \cite{waldspurger2015phase}.
This problem has a structure similar to the
classical MAXCUT relaxation, which is the origin of the name PhaseCut.
Unfortunately, the resulting SDP program should be solved using interior point methods, which are quite computationally expensive \cite{waldspurger2015phase}.

%For example,
% a pre-trained network is  used as a regularization of the optimization problem \cite{metzler2018prdeep,hand2018phase,hyder2020solving},
%  or an untrained neural network is used as a structural prior to restrict the feasible
%set \cite{?}.
%%Supervised training is usually used using matched data set.
%%$\{ \hat{\yb}^{(m)}, \xb^{(m)}\}_{m=1}^M$, where $M$ is the number of training data. 
%Specifically, in \cite{metzler2018prdeep}, the authors  leveraged the regularization-by-denoising (RED) framework \cite{?} and a convolutional neural network denoiser,
%whereas in \cite{icsil2019deep}, and iterative hybrid  of neural networks and the model-based HIO method was proposed. The authors in \cite{uelwer2019phase} proposed to employ conditional generative adversarial network (cGAN) to restore the phase information of the signal.

\subsection{Deep Learning Approaches}
Deep learning approaches for phase retrieval studies can be broadly divided into two categories based on how the neural network is utilized to solve the problem.

In the first category, a pretrained network is employed as a regularization or prior for the reconstruction \cite{metzler2018prdeep,hand2018phase,hyder2020solving}.
Specifically, in \cite{metzler2018prdeep}, a pretrained network for the denoising task serves as a regularization when exploiting the regularization-by-denoising (RED) framework to solve the phase retrieval problem. Hand et al \cite{hand2018phase} proposed  generative prior from a pretrained generative model to restore phase information. 
In \cite{hyder2020solving},  the reference provided by the trained network was exploited as a prior during the reconstruction process.
Although leveraging the knowledge from the pretrained network leads to a better reconstruction quality,  the optimization problems defined in these algorithms have to be solved in an iterative manner. %  a high computational complexity is required.

In the second category, a network is trained in an end-to-end manner to restore the signal from the given measured magnitude  through supervised learning scheme \cite{icsil2019deep, uelwer2019phase}.
For example, for a large number of paired magnitude and label data set $(\bb,\xb)$, the neural network $\Qc_\Theta$ parameterized by $\Theta$ 
is trained by minimizing the $l_2$ loss:
\begin{align}\label{eq:super}
\min_\Theta E\| \xb - \Qc_\Theta(\bb)\|^2,
\end{align}
where $E[\cdot]$ denotes the mean with respect to the training data set.
Once the neural network is trained with the estimated parameter $\Theta^*$, at the inference phase  the
image estimate can be obtained by $\Gc_{\Theta^*}(\bb)$ for a given magnitude data $\bb$.
% Here, the networks were trained using a paired dataset, consisting of label images and corresponding magnitudes of Fourier transform.
%To improve performance, 

Although these algorithms outperform the traditional alternating projection-based methods, obtaining a large amount of matched data for training is often difficult in practice.
Moreover, most existing variations in this category aim at improving the performance of the neural network using various tricks, as their feed-forward neural networks are not sufficient for high quality reconstruction. For example, in \cite{icsil2019deep}, a hybrid use of two networks and the HIO method was proposed to address the phase retrieval problem. Uelwer et al \cite{uelwer2019phase} proposed using the Fourier measurement and the latent variable together as input to train the conditional generative adversarial network (GAN) \cite{isola2017image}. 

Yet another important technical issue is
that the supervised training in \eqref{eq:super} is very sensitive  to the ambiguity
within the equivalent class as stated in \eqref{eq:equiv}. % for a given magnitude measurement.
This is because the correct loss function in \eqref{eq:super} should be the distance from the equivalent
class rather than a specific label, but the supervised learning frameworks usually  aim at fitting to a particular label, leading to the fundamental limitations \cite{manekar2020}.
To address this problem, the authors in \cite{manekar2020} proposed symmetry-breaking learning approach
by minimizing the following loss function:
\begin{align}\label{eq:simb}
\min_\Theta E \| \bb - |\Ab \Qc_\Theta(\bb)|\|^2,
\end{align}
which is in fact minimizing the data consistency term  \cite{zhu2018image}.

\begin{figure*}[!t] 	
\center{ 
\includegraphics[width=18cm]{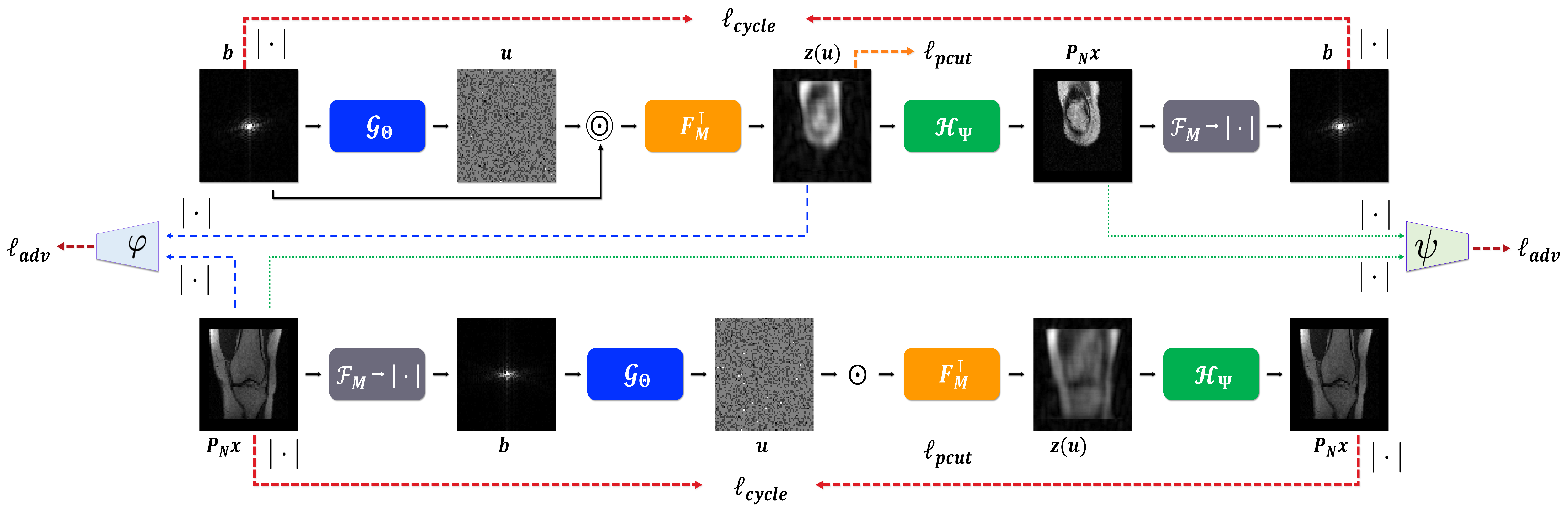}
}
\caption{Overview of the DeepPhaseCut training scheme. $\Pb_N$ is a zero-padding matrix converting $\xb \in \Cd^N$ to a zero-padded $M$-dimensional matrix $\Pb_N \xb \in \Cd^M$. $\Fb_M$ denotes $M \times M$-DFT matrix. 
There are two generators $\Gc_\Theta$, $\Hc_\Psi$ and two discriminators $\varphi$, $\psi$.
$\Gc_\Theta$ and $\Hc_\Psi$ refer to the phase network and image refinement network.
The network employed three losses -- adversarial loss, PhaseCut loss, and cycle consistency loss. Here, $|\cdot|$ refers to  the element-wise magnitude.}
\label{fig:flowchart}
\end{figure*}

\section{DeepPhaseCut}
\label{sec:theory}

\subsection{Overall framework}

In this section, we derive DeepPhaseCut algorithm as a feed-forward deep neural network approach that is trained in an unsupervised manner to  address the ambiguity
of the solutions. 

%In particular, 
%Given the PhaseCut loss function, we can easily derive the neural network architecture for the phase retrieval problem.
Specifically, similar to the alternating projection methods in \eqref{eq:s1}-\eqref{eq:s3},  our DeepPhaseCut algorithm is composed of the following steps:
\begin{align}
\ub &=\Gc_\Theta(\bb) \label{eq:u} \\
\zb(\ub) &= \Fb_M^\top (\bb\odot\ub) \label{eq:z} \\
\xb &=\Hc_\Psi(\zb(\ub)) \label{eq:z}
\end{align}
where $\Gc_\Theta$ is the phase estimation neural network parameterized by $\Theta$, and $\Hc_\Psi$ is the image refinement network parameterized by $\Psi$.
In the following, we will explain how these two networks can be trained.

\subsection{Loss functions}

\subsubsection{PhaseCut loss}

One of the important contributions of our method is the discovery of PhaseCut loss, which is useful
for the phase network training.
This loss function arises from the PhaseCut cost in \eqref{eq:lM}, but the backpropagating the original loss function in \eqref{eq:lM} is difficult
due to the subspace formulation. Therefore, our loss is derived so that 
it can be used for backpropagation for the phase network training.
%The following Proposition provides a closed form loss function.
\begin{proposition}\label{prp:pcut}
The PhaseCut cost in \eqref{eq:lM} can be equivalently represented by
\begin{align*}
\ell_{pcut}(\ub) &= \ub^\top\Mb\ub = \|\zb_{\Lambda^c}(\ub)\|_2^2,
\end{align*}
where $\zb(\ub)$ is the inverse DFT of  the estimate spectrum with respect to the phase estimate $\ub$, i.e.
\begin{align}
\zb(\ub):=  \Fb_M^\top (\bb\odot \ub),
\end{align}
and $\Lambda^c$ denotes the indices outside of the non-zero support of the ground-truth image,
and $\zb_{\Lambda^c}$ is a subvector on the index set of $\Lambda^c$.
  \end{proposition}
\begin{proof}
Using the definition of $\Mb$ in \eqref{eq:M}, 
the cost function in \eqref{eq:lM} is given by
\begin{align*}
\ell_{pcut}(\ub) &= \ub^\top\Mb\ub \notag\\
&=  \ub^\top\mathrm{diag}(\bb)(\Ib -\Ab\Ab^\dag )  \mathrm{diag}(\bb)\ub\notag\\
&= \widehat\zb(\ub)^\top(\Ib -\Ab\Ab^\dag )\widehat\zb(\ub)
\end{align*}
where
\begin{align*}
\widehat\zb(\ub):=\mathrm{diag}(\bb)\ub=\bb\odot\ub.
\end{align*}
Furthermore, using the representation in \eqref{eq:A} of  the oversampled DFT matrix $\Ab$,
we have
\begin{align*}
\Ab\Ab^\dag  & = \Fb_M \Pb_{N} \left(\Pb_N^\top \Fb_M^\top \Fb_M \Pb_N\right)^{-1} \Pb_N^\top \Fb_M^\top\\
&=\Fb_M \Pb_{N} \Pb_N^\top \Fb_M^\top,
\end{align*}
where we use  $\Fb_M^\top \Fb_M=\Ib$  for the DFT matrix and $\Pb_N^\top\Pb_N=\Ib$ for the projection matrix.
Accordingly, the phase cut cost function becomes
\begin{align*}
\ell_{pcut}(\ub) 
%&= \widehat\xb(\ub)^\top(\Ib -\Ab\Ab^\dag )\widehat\xb(\ub)\\
&= \widehat\zb(\ub)^\top(\Ib-\Fb_M \Pb_{N} \Pb_N^\top \Fb_M^\top)\widehat\zb(\ub)\\
&=\zb(\ub)^\top (\Ib- \Pb_{N} \Pb_N^\top) \zb(\ub) ,
%&= \|\xb_{\Lambda^c}(\ub)\|_2^2
\end{align*}
where  we use
$$\zb(\ub)=\Fb_M^\top\widehat\zb(\ub) =   \Fb_M^\top (\bb\odot \ub) . $$
 Furthermore, it is easy to see that $\Pb_{N} \Pb_N^\top$ is a diagonal
matrix with 1 at the non-zero support indices $\Lambda$.  Therefore, we have
\begin{align*}
\ell_{pcut}(\ub) 
%%&= \widehat\xb(\ub)^\top(\Ib -\Ab\Ab^\dag )\widehat\xb(\ub)\\
%&= \widehat\xb(\ub)^\top(\Ib-\Fb_M \Pb_{N} \Pb_N^\top \Fb_M^\top)\widehat\xb(\ub)\\
%&=\xb(\ub)^\top (\Ib- \Pb_{N} \Pb_N^\top) \xb(\ub) \\
&= \|\zb_{\Lambda^c}(\ub)\|_2^2,
\end{align*}
where $\Lambda^c$ denotes the complement set of $\Lambda$.
\end{proof}
Proposition~\ref{prp:pcut} 
implies that PhaseCut loss is actually the  energy leakage term beyond the image support  so that its minimization
results in a phase recovery that meets the support constraint.
PhaseCut loss should be minimized not only for the original reconstruction using Fourier magnitude measurements, but also for reconstruction
from any synthesized
Fourier magnitude measurement if the phase network is trained correctly. This leads to the following composite PhaseCut loss:
\begin{align}
\label{eq:pcut_loss}
\ell_{pcut}(\Theta) := \|\zb_{\Lambda^c}(\Gc_\Theta(\bb))\|^2 + \|\zb_{\Lambda^c}(\Gc_\Theta(|\Ab\xb|)\|^2
\end{align}
where $\bb$ is the actual Fourier magnitude measurement, and $\xb$ is any image in the training data.

%implies that the phase cut loss is in fact  the out-of-the support energy leakage term, and its minimization
%leads to the phase recovery that satisfies the support constraint.

\subsubsection{CycleGAN loss}

In addition to the support constraint from PhaseCut loss,  we employ the symmetry breaking loss from \eqref{eq:simb}:
\begin{align}
\label{eq:symb2}
\ell_{symb}(\Theta, \Psi) 
& =  \| \bb - |\Ab \Hc_\Psi(\zb(\Gc_\Theta(\bb)) |\|^2. 
\end{align}
Note that the symmetry breaking loss is indeed the self-consistency condition for the Fourier
magnitude measurement. This leads to the additional self-consistency terms
for the image. More specifically, the image domain self-consistency can be imposed by minimizing the following loss:
\begin{align}
\ell_{cst}(\Theta, \Psi) = &~ \| |\xb| - |\zb(\Gc_\Theta(|\Ab\xb|))| \|^2 \nonumber \\
&+ \| |\xb| - |\Hc_\Psi(\zb(\Gc_\Theta(|\Ab\xb|)))| \|^2, \label{eq:self}
\end{align}
which implies that  the synthesized images from the Fourier magnitude with respect to any image $\xb$ should be as close as possible
to the original $\xb$.  In fact,  the first and second term in \eqref{eq:self} impose the self-consistency to the image from the phase network and the image refinement
network, respectively.
If we sum the self-consistency terms in both image and measurement domains, we can obtain a cycle-consistency loss for our network:
\begin{align}
\label{eq:cycle_loss}
\ell_{cycle}(\Theta, \Psi) =& \| |\xb| - |\Hc_\Psi(\zb(\Gc_\Theta(|\Ab\xb|)))| \|^2 \\
& + \| |\xb| - |\zb(\Gc_\Theta(|\Ab\xb|))| \|^2 \nonumber \\
& +  \| \bb - |\Ab \Hc_\Psi(\zb(\Gc_\Theta(\bb)) |\|^2.  \nonumber
\end{align}

%we further need to impose a reconstruction loss to improve the performance of both phase and reconstruction network.

Next consideration is the reconstruction loss.
Recall that there exists ambiguity in the phase retrieval problem up to the equivalent class, so that
it is not a good idea to use the sample-wise distances like $l_1$ or $l_2$ as the reconstruction loss.
Instead,
%instead of $l_2$ distances, 
we are interested in minimizing the distance between the distributions of the generated samples and their equivalent classes.
This can be done using adversarial loss \cite{goodfellow2014generative}, and we employ the least squares GAN (LS-GAN) loss \cite{mao2017least} as our adversarial loss.
Specifically, we employ the LS-GAN loss for both phase and refinement networks that generate
 the intermediate and final image, respectively.
 The corresponding adversarial losses are given  for the discriminator update by:
 \begin{align}
\label{eq:GAN_loss}
&\ell_{phase\_adv}(\varphi):=\\
%&  \int_\Bc \left( \varphi(|\zb(\Gc_\Theta(\bb))|) - 1 \right)^2 d\nu(\bb) \nonumber  \\
&   \int_\Xc \left( \varphi(|\xb|) - 1 \right)^2 d\mu(\xb) + \int_\Bc  \varphi(|\zb(\Gc_\Theta(\bb))|)^2 d\nu(\bb),  \nonumber \\
&\ell_{refine\_adv}(\psi):=\\
%&   \int_\Bc \left( \varphi_\Xi(\Sc \Hc_\Psi(\zb(\Gc_\Theta(\bb)))) - 1 \right)^2 d\bb \nonumber \\
&   \int_\Xc \left( \psi(|\xb|) - 1 \right)^2 d\mu(\xb) + \int_\Bc  \psi(|\Hc_\Psi(\zb(\Gc_\Theta(\bb))|)^2 d\nu(\bb), \nonumber 
\end{align}
and  for the generator update by:
 \begin{align}
\ell_{phase\_adv}(\Theta):=&
  \int_\Bc \left( \varphi(|\zb(\Gc_\Theta(\bb))|) - 1 \right)^2 d\nu(\bb),\\
%&   \int_\Xc \left( \varphi(|\xb|) - 1 \right)^2 d\mu(\xb) + \int_\Bc  \varphi(|\zb(\Gc_\Theta(\bb))|)^2 d\nu(\bb)  \nonumber \\
\ell_{refine\_adv}(\Theta,\Psi):=&
   \int_\Bc \left( \psi(|\Hc_\Psi(\zb(\Gc_\Theta(\bb)))|) - 1 \right)^2 d\nu(\bb),
%&   \int_\Xc \left( \psi(|\xb|) - 1 \right)^2 d\mu(\xb) + \int_\Bc  \psi(|\Hc_\Psi(\zb(\Gc_\Theta(\bb))|)^2 d\nu(\bb) \nonumber 
\end{align}
where we assume that the target image space $\Xc$ is equipped with a probability measure $\mu$, and
the measurement space  $\Bc$ is  with a probability measure $\nu$. 

Putting them together, the total loss function is given by:
\begin{align}
&\ell_{total}(\Theta, \Psi; \varphi,\psi)\notag \\
& :=   \ell_{adv}(\Theta, \Psi; \varphi,\psi) + \kappa \ell_{cycle}(\Theta, \Psi) + \rho, \ell_{pcut}(\Theta), \label{eq:loss_DeepPhaseCut}
\end{align}
where $\kappa$ and $\rho$ are hyper-parameters, and
\begin{align}
 \ell_{adv} = \ell_{phase\_adv}+\ell_{refine\_adv}.
\end{align}
The overall network can be trained by the following minimization problem:
\begin{align}
\min\limits_{\Theta,\Psi,\varphi,\psi}\ell_{total}(\Theta, \Psi; \varphi,\psi).
\end{align}
Note that the first two terms in \eqref{eq:loss_DeepPhaseCut} is similar to the cycleGAN loss,
although the generator for the forward mapping is replaced by a deterministic Fourier transform followed by the magnitude operation.
In fact, the mathematical origin of this structure is studied in our recent work on optimal transport driven cycleGAN \cite{sim2019optimal}, which
shows that the minimizing the statistical distances simultaneously in both measurement and image domain  leads to the
cycleGAN structure as a dual Kantorovich formulation.
The overall flowchart of DeepPhaseCut is shown in Fig. \ref{fig:flowchart}.

\section{Methods}
\label{sec:method}

\subsection{Datasets}

To verify the versatility and improved performance of the proposed method, we conducted experiments for both real-valued and complex-valued images.
%As for real-valued images, we used facial images, natural images, and microscopy images.
%Furthermore,  magnetic resonance imaging (MRI) data set was also used as  complex-valued images.
%%To the best of our knowledge, this is the first attempt to use complex-valued images for phase retrieval problem using deep learning.

\subsubsection{Facial Image}
We used CelebA dataset \cite{liu2015deep} which consists of 202,599 facial images of 10,177 different celebrities. 
We changed the RGB images to grayscale images, and resized them from the center cropped images of 108 $\times$ 108 pixels to 64 $\times$ 64 pixels. 
This preprocessing was identical to that employed in \cite{hand2018phase,uelwer2019phase}.
Among 202,599 images, 199,000 images were used for training, 1,121 images for validation, and the rest for testing. 

\subsubsection{Natural Image}
DIV2K dataset \cite{agustsson2017ntire} with 800 high-quality (2K resolution) images and BSD300 dataset \cite{MartinFTM01} were used for training. 
The RGB images were converted into grayscale images and randomly cropped images of 256 $\times$ 256 pixels were resized to 64 $\times$ 64 pixels.
To evaluate the proposed network, we used Set12, one of the most widely used benchmark datasets.

\subsubsection{MRI Image}
To validate the performance of our Fourier phase retrieval method using more realistic complex-valued data set,
we used FastMRI dataset  \cite{zbontar2018fastmri} composed of complex-valued knee images from 1,129 cases. 
We used a single coil data set, which is obtained from the multi-coil data set acquired by a clinical 1.5 T Siemens Magnetom Area, or one of the three 3T systems (Prisma, Siemens Magnetom Skyra, and Biograph mMR) using 2D TSE protocol. The acquisition parameters were as follows: the size of matrix = 320 $\times$ 320, repetition time (TR) = 2,200 $\sim$ 3,000 ms, echo time (TE) = 27 $\sim$ 34 ms, slice thickness = 3 mm, and in-plane resolution = 0.5 mm $\times$ 0.5 mm. We resized the data from the center cropped images with 320 $\times$ 320 pixels to 64 $\times$ 64 pixels.
Out of 1,129 cases, 973 sets were used for training, which corresponds to 29,877 images of training data. The remaining 156 cases were used for testing, which corresponds to 4,873 slices of test data.

\subsubsection{Microscopy Image}
As for  realistic real-valued data set,  we obtained {images of silver coated silica beads} using {a custom-built} optical microscopy.
The diameter of the microbead is about {4.5} $\mu$m.
To build a microscopic imaging system, we {used a diode pumped solid state laser emitting at 532 $n$m} as source, {the set of an} objective lens (Nikon, 10$\times$ / 0.28 NA) and {a} tube lens (f$=$200 $n$m) to magnify {the} beads, and {a} charge-coupled device (CCD) as a detector.
{For a single measurement, this imaging system captured a 10 times magnified image of the clustered beads and the original measurement size was 2,248 $\times$ 2,048.}
%With this imaging system, beads magnified 10 times can be recorded.
%Clusters of beads were collected from the image with 2,248 $\times$ 2,048 pixels. 
%We found the p where the intensity is below the threshold value of 25 because the intensity of data points is in the range of 0 to 50. 
%Approximate clusters of 120 $\times$ 120 pixels were then created in which the beads were grouped \add{from the measurements.} 
Approximate clusters of 120 $\times$ 120 pixels were {then sampled from each measurement.} 
{Subsequently, patch of 100 $\times$ 100 pixels was center cropped from each cluster.}
% Patch of 100 $\times$ 100 pixels were obtained based on the center found from each cluster of beads \add{of 120 $\times$ 120 pixels}.
Out of 20 measurements, 18 measurements were used for training which corresponds to 55,147 images {with 100 $\times$ 100 pixels}.
The remaining 2 measurements used for testing, which corresponds to 1,000 images {with 100 $\times$ 100 pixels}.
During training, the images were randomly cropped to 64 $\times$ 64 pixels, while during the test, the images were {center} cropped to 64 $\times$ 64 pixels.

\subsection{Network Architecture}

We used two neural networks $\Gc_\Theta$ and $\Hc_\Psi$ and two discriminators $\varphi$ and $\psi$ for DeepPhaseCut. 
We employed the modified U-net architecture with spatial attention module in skip connection for two networks, as shown in Fig. \ref{fig:architecture} (a). 
In particular, U-net is comprised of four stages, each stage consisting of blocks of 3$\times$3 convolution, instance normalization\cite{ulyanov2016instance}, and the retified linear unit (ReLU). 
The number of convolutional filters was set to 64 in the first stage, which reached 1024 in the last stage by doubling at each stage.
We utilized 2$\times$2 average pooling and bilinear interpolation for the pooling and unpooling operation, respectively. 
The spatial attention module in skip and concatenation was utilized to amplify the important regions in the feature maps for efficient learning \cite{tao2019improving}. For attention module, we used 3$\times$3 convolutional layer.
Next, a spatial softmax function and a max normalization operation are applied to each channel of the feature maps and  the resulting spatial attention weight matrix is multiplied to each feature map as shown in Fig. \ref{fig:architecture}(b).
The final output of network is provided after 1$\times$1 convolution (red arrow in Fig. \ref{fig:architecture}(a)).

Since $\Gc_\Theta$ is trained to estimate complex valued phase $\ub$, both the output of $\Gc_\Theta$ and the input of $\Hc_\Psi$ are complex values. To process the complex value, we stacked the real and imaginary part of the complex data in the channel dimension, which resulted in the number of channels for the output of $\Gc_\Theta$ and the input of $\Hc_\Psi$ being 2.
The number of output channels of $\Hc_\Psi$ was determined taking into account the inherent nature of the data, which resulted in 1 and 2 for real-valued and complex-valued datasets, respectively.
The complex result can be formed using the real and imaginary channels of the output.

We adopted the PatchGAN discriminator \cite{zhu2017unpaired} for our discriminators $\varphi$ and $\psi$, which consists of five convolution layers as shown in Fig. \ref{fig:architecture}(c). 
The first convolution layer is composed of 64 sets of 4$\times$4 convolution kernel, in which the number of convolutional filters in the following layers, with the exception of the last layer, is doubled by that of the previous layer. In the last convolution layer, there is a single 1$\times$1 convolution filter to provide the output for classifying whether the patches are real or synthesized by the generators.
Our discriminators process  single channel images, which are formed by the magnitude images from the complex-valued data for an efficient and stable training.

\begin{figure}[!hbt] 	
\center{ 
\includegraphics[width=8.6cm]{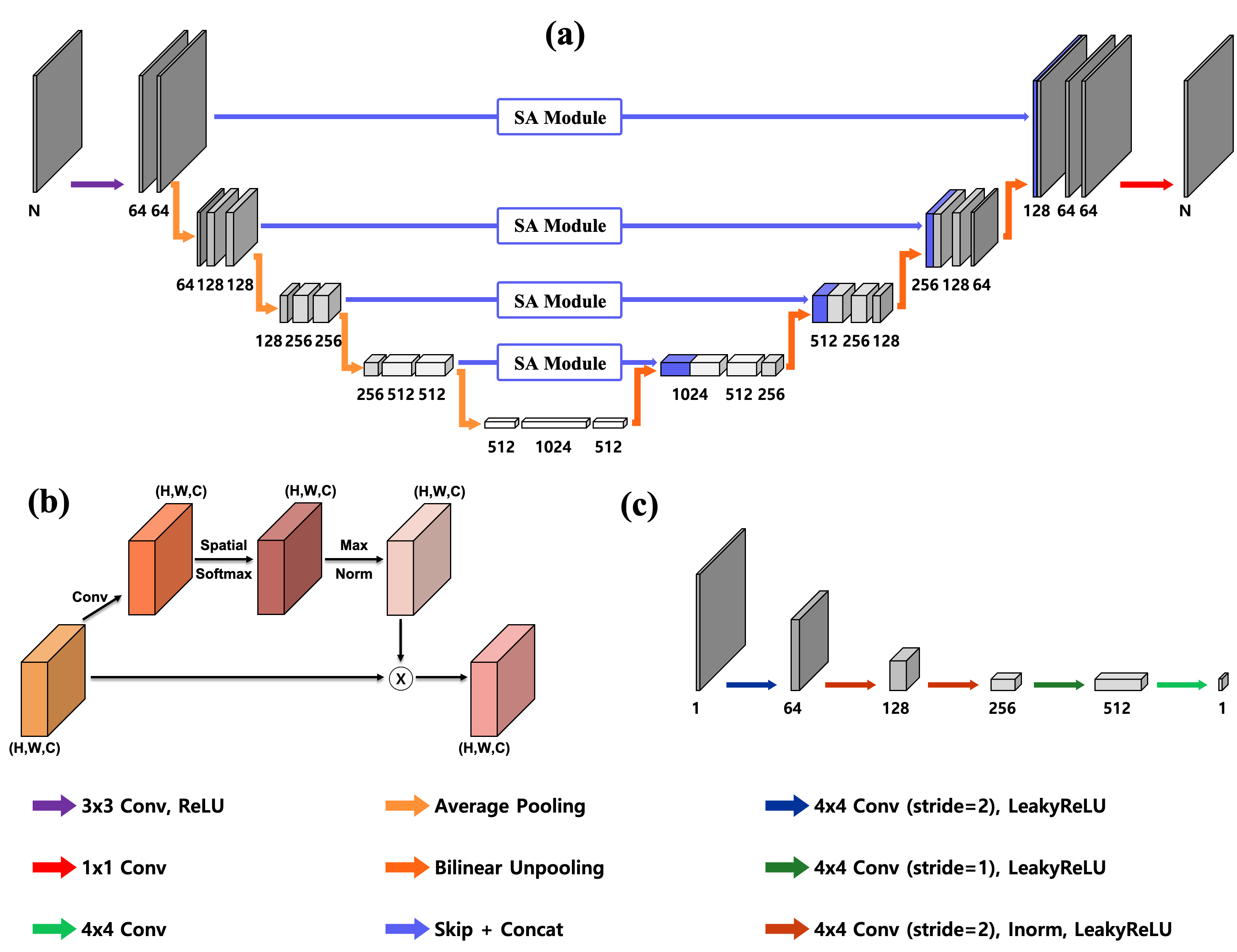}
}
\caption{(a) Modified U-net architecture for the generators $\Gc_\Theta$ and $\Hc_\Psi$. SA module denotes spatial attention module. (b) Spatial attention module \cite{tao2019improving}. (c) Network architecture of the discriminators $\varphi$ and $\psi$. }
\label{fig:architecture}
\end{figure}

\begin{figure*}[!hbt] 	
\center{ 
\includegraphics[width=16cm]{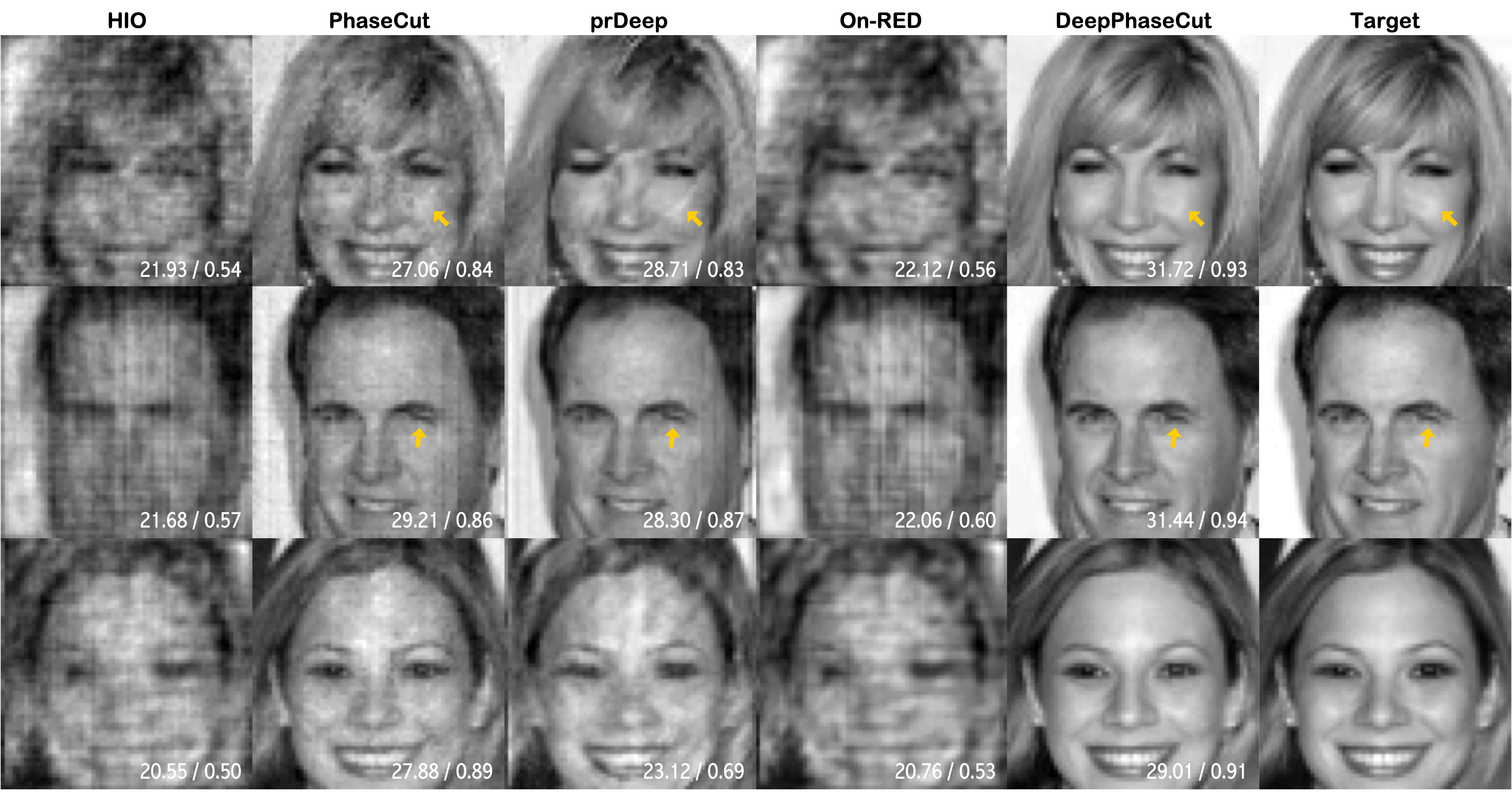}
}
\caption{CelebA data reconstruction results using HIO \cite{fienup1982phase}, PhaseCut \cite{waldspurger2015phase},  prDeep \cite{metzler2018prdeep}, On-RED \cite{wu2019online}, and proposed DeepPhaseCut. The arrows indicate remarkable parts. The values in the corners are PSNR / SSIM index for each image with respect to the target image.}
\label{fig:result_celebA}
\end{figure*}

\subsection{Network Training}

For all data sets, we employed zero-padding  to guarantee uniqueness \cite{hayes1982reconstruction,bruck1979ambiguity,bates1982fourier,shechtman2015phase}. 
In particular, the oversampling ratio was set to 4 so that the size of measured magnitude in Fourier domain was 128 $\times$ 128 pixels. 

For preprocessing, the image intensity range was set to $[0,1]$.
The proposed network was implemented in Python using PyTorch library \cite{paszke2017automatic} and trained using NVidia GeForce GTX 1080-Ti graphics processing unit. Adam optimizer \cite{kingma2014adam} was used for the optimization with parameters $\beta_1 = 0.5$ and $\beta_2 = 0.999$. The update rate of the generators and discriminators was set to $1:1$ at each step. 
The specific details for various applications are as follows.

\subsubsection{Facial Images}
For CelebA dataset, the hyper-parameters $\kappa$ and $\rho$ in \eqref{eq:loss_DeepPhaseCut} were set to 20.
The network was trained for 45 epochs. 
The initial learning rate was specified to $5\times 10^{-7}$ for 20 epochs. Then the learning rate for the remaining 25 epochs was linearly reduced to $3.5 \times 10^{-7}$.
The mini-batch size was 1. 
The training using the CelebA dataset took about ten days.

\subsubsection{Natural Images}
With the DIV2K and BSD300 data sets, we set the hyper-parameters in \eqref{eq:loss_DeepPhaseCut} to $\kappa=200$ and $\rho=30$, respectively.
The proposed model was trained for 240 epochs, with the learning rate specified to $1 \times 10^{-3}$ and then linearly decreased to 0 for the next 200 epochs.
The network was trained for seven days.

\subsubsection{MRI Images}
To evaluate the proposed network with complex-valued data, we set the hyper-parameters in \eqref{eq:loss_DeepPhaseCut} to $\kappa=\rho=20$. 
We trained the model for 10 epochs with the learning rate $1 \times 10^{-5}$. Then, the learning rate was linearly reduced to 0 for the next 150 epochs.
The training with the fastMRI dataset took about six days.

\subsubsection{Microscopy Images}
To train the proposed network using microscopy images of beads, the hyper-parameters $\kappa$ and $\rho$  in \eqref{eq:loss_DeepPhaseCut}  were set to 200 and 20, respectively.
The initial learning rate was $1 \times 10^{-4}$ for 40 epochs. At 120 epochs of training, the learning rate was decayed to $2 \times 10^{-5}$. The training of the proposed model for microscopy dataset took about seven days.

\subsection{Comparative Studies}
For comparison, HIO \cite{fienup1982phase, fienup1987reconstruction} was used as a representative alternating projection approach for comparison. The hyper-parameter in \eqref{eq:HIO} was set to $\beta=0.9$.
The HIO method ran 10 times from random initialization, and 50 iterations were performed each time. The reconstruction with the smallest residual ($\| \bb - |\Ab\xb_{HIO}| \|_2$) among the 10 estimated signals was used as an initialization for HIO. 
Then, the final reconstruction was obtained by  running 1000 iterations from the initialization.
When applying HIO to real-valued data, non-negativity and support are used as spatial domain constraints \cite{fienup1982phase}. Meanwhile, when applying HIO to complex-valued data, only the support is utilized as spatial domain constraint \cite{fienup1987reconstruction}. 

PhaseCut \cite{waldspurger2015phase} was employed for comparison. We used PhaseCutToolbox, which was provided by the authors in \cite{fogel2016phase} and is available \href{http://www.di.ens.fr/~aspremon}{http://www.di.ens.fr/$\sim$aspremon}. 
The hyper-parameters for solving the SDP program derived from \eqref{eq:lM} were followed as indicated in the toolbox.
As implemented in the PhaseCutToolbox, HIO was performed 25 times on the image reconstructed with PhaseCut.

We also used prDeep \cite{metzler2018prdeep} and online-RED (On-RED) \cite{wu2019online} as a comparison, which are the stat-of-the-art deep learning-based algorithms. 
For prDeep, we utilized a public implementation provided by the authors in \cite{metzler2018prdeep}, which is available at \href{https://github.com/ricedsp/prDeep}{https://github.com/ricedsp/prDeep}. The implementation details were taken from the original paper \cite{metzler2018prdeep}.
PrDeep was only employed for comparison of CelebA, Set12, and microscopy data sets, since the algorithm was proposed for the phase retrieval of signals with real value.
For On-RED, we used the algorithm implemented by the authors, which was provided at \href{https://github.com/wustl-cig/2019-ICCVW-OnlineRED}{https://github.com/wustl-cig/2019-ICCVW-OnlineRED}.
The performance of On-RED is heavily dependent on one of the hyper-parameters $\tau$, which determines the step size $\gamma$. Therefore, $\tau$ was tested with 0.005 steps between 1 and 4, and we set $\tau$ to 2.950 which gives the best performance. We employed On-RED to compare the performance using the CelebA dataset.
The final HIO reconstruction was utilized as an initialization of prDeep and On-RED.

%Furthermore, as the baseline algorithm, the cycleGAN with an additional pair of generator and discriminator for the domain $\Yc$ was used.
%In particular, the deterministic operator $\Hc$ in the proposed method was replaced by the additional generator.

The peak signal-to-noise ratio (PSNR) and the structural similarity (SSIM) index \cite{wang2004image} were used to provide a quantitative evaluation.
For the fastMRI dataset, since the reconstruction is a complex-valued image, the metrics are calculated using the magnitude images.
The PSNR is calculated using mean squared error ($MSE$) as follows:
\begin{eqnarray}
PSNR &=& 20 \cdot \log_{10} \left(\dfrac{MAX_{\xb}}{\sqrt{MSE(\tilde{\xb}, \xb)}}\right), 
\label{eq:psnr}		 
\end{eqnarray}
where $\tilde{\xb}$ denotes a reconstructed image and $\xb$ refers to the ground-truth image. 
$MAX_{\xb}$ is the maximum pixel value of $\xb$.
The SSIM index is defined by \cite{wang2004image}:
\begin{equation}
SSIM = \dfrac{(2\mu_{\tilde{\xb}}\mu_{\xb}+c_1)(2\sigma_{\tilde{\xb}\xb}+c_2)}{(\mu_{\tilde{\xb}}^2+\mu_{\xb}^2+c_1)(\sigma_{\tilde{\xb}}^2+\sigma_{\xb}^2+c_2)},
\end{equation}
where $\mu_{\xb}$, $\sigma_{\xb}^2$, and $\sigma_{\xb_1\xb_2}$ denote average and variance of $\xb$, and covariance of $\xb_1$ and $\xb_2$, respectively. 
$c_1$ and $c_2$ are $(k_1R_\xb)^2$ and $(k_2R_\xb)^2$, respectively, where $R_\xb$ is the dynamic range of pixel values of $\xb$. We used the default values $k_1 = 0.01$ and $k_2 = 0.03$.

\begin{figure*}[!t] 	
\center{ 
\includegraphics[width=14cm]{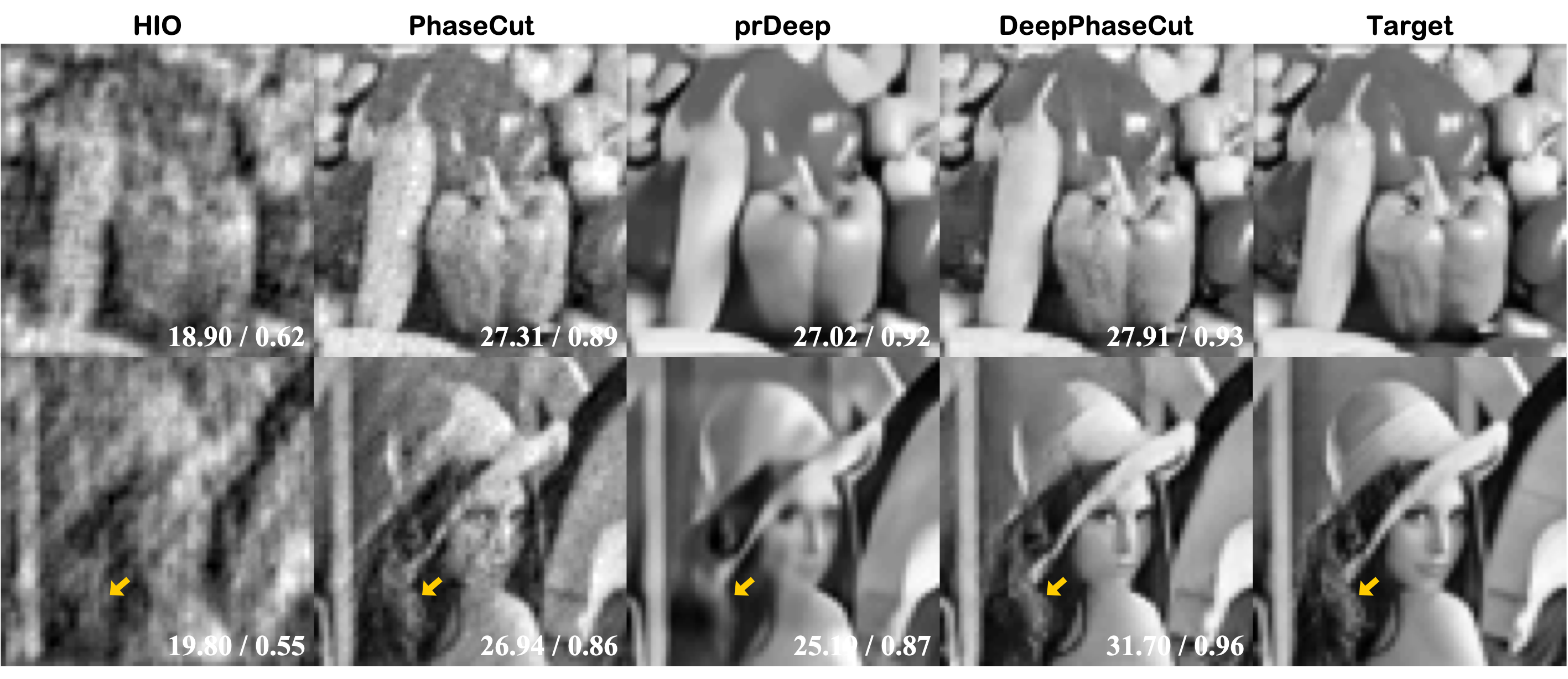}
}
\caption{Natural image reconstruction results using HIO \cite{fienup1982phase}, PhaseCut \cite{waldspurger2015phase},  prDeep \cite{metzler2018prdeep}, and proposed DeepPhaseCut. The values in the corners are PSNR / SSIM index for each image with respect to the target image.}
\label{fig:result_nr}
\end{figure*}

\section{Experimental Results}
\label{sec:experiments}
\subsection{Facial Images}
First, we evaluated the proposed model using the CelebA dataset which consists of real-valued images. 
To verify the feasibility of the proposed method, the performance of the proposed method was compared with those of HIO \cite{fienup1982phase}, PhaseCut \cite{waldspurger2015phase}, prDeep \cite{metzler2018prdeep}, and On-RED \cite{wu2019online}.
As shown in the first column of Fig. \ref{fig:result_celebA}, HIO produced very blurry and noisy results where it is difficult to distinguish between structures such as eyes, nose and lips.
PhaseCut usually produced better results than HIO, making it easier to identify the eyes, nose, mouth, etc. on the face. However, the texture of the face was  not  often properly restored.
As depicted in third column of Fig. \ref{fig:result_celebA}, although prDeep provided better reconstruction results by removing many unnatural artifacts, noisy patterns were still present in most of the image. 
%Even though there are no noisy patterns and artifacts in the reconstruction, the results do not resemble the ground-truth images. 
Even though the HIO reconstruction was utilized as the initialization for On-RED, On-RED provided results that were almost similar to the results using HIO. 
This may be because  On-RED was proposed for phase retrieval from coded diffraction patterns (CDP) \cite{candes2013phaselift},
%general phase retrieval problems, 
rather than for the Fourier phase retrieval;  therefore, we do not carry out the comparison with On-RED using other data sets.
As illustrated in the fifth column of Fig. \ref{fig:result_celebA}, the proposed DeepPhaseCut provided the best reconstruction results compared to other algorithms.
%Specifically, the smoothness and artifacts are not demonstrated in the reconstruction using the proposed method.
In particular, the detailed structures such as the shape of eyebrows and wrinkles can be restored closely following the ground-truth image.

\begin{table}[!hbt]
 \begin{center}
 	\resizebox{0.4\textwidth}{!}{
 \begin{tabular}{c|ccc}
%  \toprule
\hline
  & PSNR [dB] & SSIM & Runtime [sec]\\
%\midrule
%\midrule
\hline\hline
HIO & 19.573 & 0.472 & 1.852\\
PhaseCut & 25.362 & 0.760 & 1044.720\\
prDeep & 26.057 & 0.738  & 33.998 \\
On-RED & 19.796 & 0.494 & 2.732\\
DeepPhaseCut & \textbf{27.119} & \textbf{0.854} & 0.048 \\
%\bottomrule
\hline
 \end{tabular}
}
 \end{center}
 \caption{Quantitative comparison of various algorithms using the CelebA dataset. The results are average values from 2,478 test images.}
 \label{table:result_celeba}
 \end{table}

In Table \ref{table:result_celeba}, the average performance of the algorithms for 2,478 images are given for quantitative comparison at the test phase. 
The proposed DeepPhaseCut significantly outperformed HIO, PhaseCut, prDeep, and On-RED in terms of PSNR and SSIM index.
In particular, our method outperformed other methods by 1.06 $\sim$ 7.55 dB in terms of PNSR.
Furthermore, as shown in Table  \ref{table:result_celeba}, the proposed DeepPhaseCut is computationally more efficient than other algorithms.
For example, several runs with different initializations are required to obtain proper results for HIO and prDeep, since they are very sensitive to the initialization. This leads to an increase in the runtime for the reconstruction using HIO and prDeep.
Furthermore, PhaseCut is very computationally expensive due to the interior point methods used to solve SDP program.
This confirmed that the proposed method not only offers better quality of the reconstruction, but also has a computational advantage over other algorithms that are compared here.

\subsection{Natural Images}
We also carried out experiments with the dataset of natural images. The trained model using DIV2K and BSD300 data sets was evaluated using Set12 dataset. 
As shown in Fig. \ref{fig:result_nr}, the proposed DeepPhaseCut restored the phase information properly, with the reconstructions being very similar to the target images. 
Compared to the results with PhaseCut or prDeep, DeepPhaseCut produces more realistic textures and sharp edges. 
As indicated by the arrows in Fig. \ref{fig:result_nr}, the tails of the hat can better be demonstrated in the results using DeepPhaseCut than using HIO, PhaseCut, or prDeep. 
More specifically, only the approximate structure can be identified without details present in the reconstruction using HIO. In the results with PhaseCut, the noise pattern is present overall, while there is no specific noise pattern in the reconstruction with DeepPhaseCut.
Even though the loss function to train the network originated from the optimization problem of PhaseCut, the proposed DeepPhaseCut is composed of not only the phase network $\Gc_\Theta$, but also the image refinement network $\Hc_\Psi$. This leads to the improved performance for phase retrieval problem.
Through the visual comparison as well as quantitative comparison using PSNR and SSIM index expressed in Fig. \ref{fig:result_nr}, the proposed DeepPhaseCut significantly outperforms the comparison algorithms.
Specifically, the reconstructed image using the proposed method shows 0.89 $\sim$ 11.90 dB gain in terms of PSNR over HIO, PhaseCut, and prDeep.

\subsection{MR Images}

To verify the applicability of the proposed method to complex-valued signal, we performed the experiment using the fastMRI dataset.
The proposed method was compared to HIO \cite{fienup1987reconstruction} and PhaseCut, since they can use complex valued data.

\begin{figure}[!hbt] 	
\center{ 
\includegraphics[width=9cm]{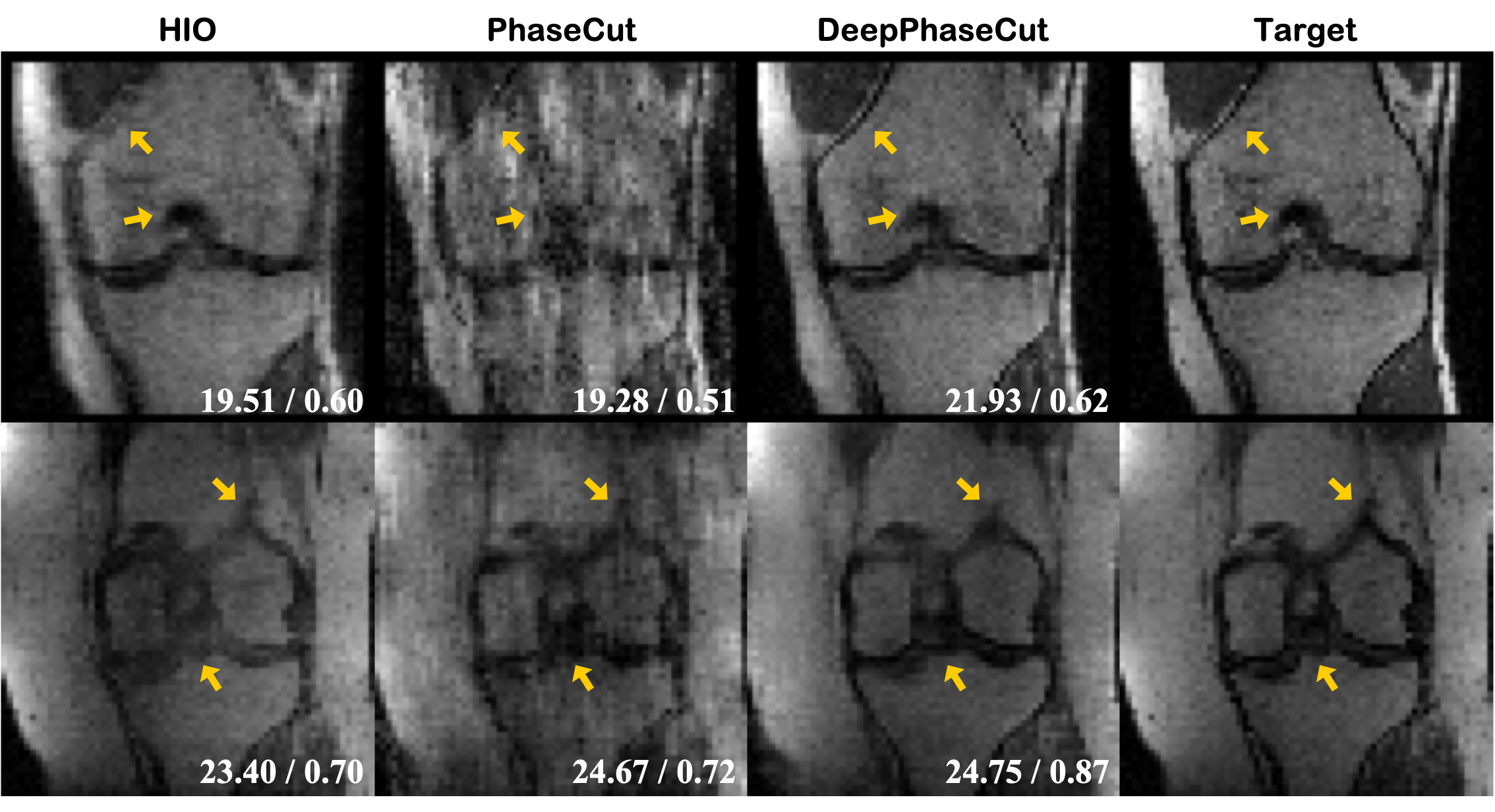}
\vspace*{-0.5cm}
}
\caption{FastMRI reconstruction results using HIO \cite{fienup1987reconstruction}, PhaseCut \cite{waldspurger2015phase}, and proposed DeepPhaseCut. The values in the corners are PSNR / SSIM index for individual image with respect to the ground-truth image.}
\label{fig:result_fastmri}
\end{figure}

\begin{figure*}[!hbt] 	
\center{ 
\includegraphics[width=14cm]{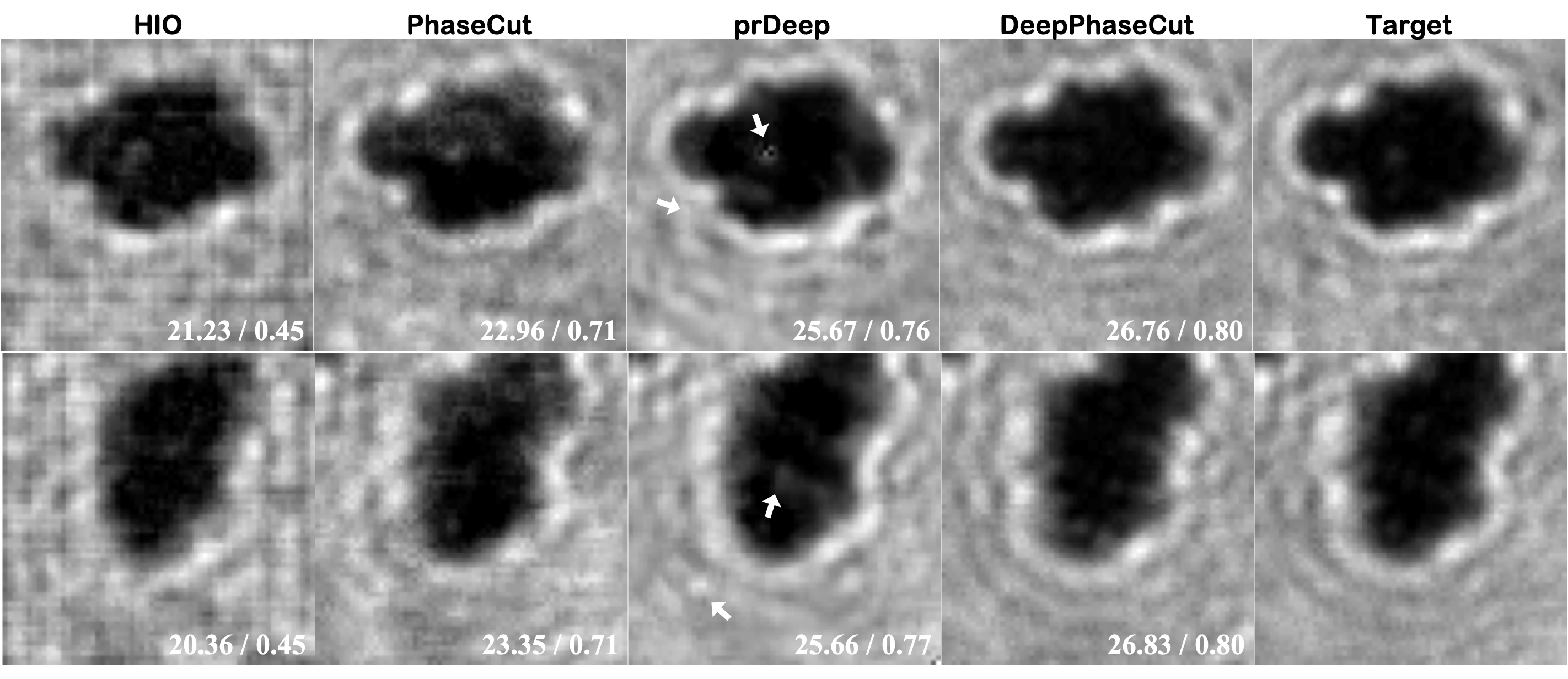}
}
\caption{Reconstructed results using HIO \cite{fienup1982phase}, PhaseCut \cite{waldspurger2015phase},  prDeep \cite{metzler2018prdeep}, and proposed DeepPhaseCut. The numbers written in the images are the corresponding PSNR / SSIM index with respect to the target image.}
\label{fig:result_micro}
\end{figure*}

Here, the first column of Fig. \ref{fig:result_fastmri} contains the results using HIO, where the support was utilized as the spatial constraint. 
Results reconstructed with HIO tended to be blurry and the details of the images were not clearly restored, while the proposed method provided better spatial resolution and the results similar to the target images.
By comparing the second and third column in Fig. \ref{fig:result_fastmri}, we found that the proposed method significantly outperformed PhaseCut. 
Specifically, PhaseCut produced artificial structures as indicated by the arrows which was not visible in the results using DeepPhaseCut.
The loss function for training the proposed DeepPhaseCut is composed of not only  $\ell_{pcut}$, which is  derived from PhaseCut, but also $\ell_{adv}$ and $\ell_{cycle}$, which can address the ambiguity of the phase retrieval problem. Therefore, the improvement in the performance can be achieved with DeepPhaseCut.
Furthermore, the reconstructed image using the proposed method is about 0.08 $\sim$ 2.35 dB gain over HIO and PhaseCut.

\begin{table}[!hbt]
 \begin{center}
	\resizebox{0.4\textwidth}{!}{
 \begin{tabular}{c|c|c|c}
  \hline
  & PSNR [dB] & SSIM index & Runtime [sec]\\
\hline\hline
%\midrule
HIO & 18.806 & 0.301 & 15.648\\
PhaseCut & 18.217 & 0.395 & 1390.480 \\
DeepPhaseCut & \textbf{19.820} & \textbf{0.399} & 0.042 \\
\hline
 \end{tabular}
}
 \end{center}
 \caption{Quantitative comparison of various reconstruction algorithms using the fastMRI datset. The results are mean values from the test set, which corresponds to 4,873 images of 156 subjects.}
 \label{table:result_fastmri}
 \end{table}

Table \ref{table:result_fastmri} shows average results for the quantitative comparison using 4,873 images.
The PSNR value of the reconstruction using the proposed method increased up 1.01 dB and 1.60 dB compared to the HIO and PhaseCut, respectively, in spite of significantly faster computational time.
%This is because our method restores phase information using a data-driven network, while HIO and PhaseCut only exploit spatial support or low-rankness in the higher dimension as the constraint.
This gain in performance obviously confirmed that the proposed method can provide superior results than the competing algorithms for the complex-valued data.

\subsection{Microscopy Images}
%To establish the improved performance of the proposed DeepPhaseCut over competing algorithms, we conducted the experiment to address the phase retrieval for microscopic images.
Fig. \ref{fig:result_micro} shows reconstruction results using HIO, PhaseCut, prDeep and DeepPhaseCut from the real microscopy data.
%Due to the high ill-posedness of phase retrieval, 
The artificial structures were generated in the reconstruction using HIO, and 
the deformation of the structure was reduced in the results reconstructed with PhaseCut.
However, there were still artifacts that make it difficult to identify the distribution of microbeads.
Thanks to the use of the pretrained network for the regularization term, the results with prDeep were better than with HIO and PhaseCut.
When reconstructing with prDeep, over-smoothing artifact is often reported, while the results with DeepPhaseCut provide nearly similar details and contrast to the target images.
%The proposed method is trained to consider the ambiguity of the phase retrieval problems up to the same class, resulting in the proper reconstruction without structural deformation and texture modification.

\begin{table}[!hbt]
 \begin{center}
\resizebox{0.4\textwidth}{!}{
	 \begin{tabular}{c|c|c|c}
\hline
  & PSNR [dB] & SSIM index & Runtime [sec]\\
\hline\hline
%\midrule
HIO & 19.691 & 0.433 & 2.244 \\
PhaseCut & 22.409 & 0.686 & 1045.812\\
prDeep & 23.849 & 0.650  & 39.592 \\
DeepPhaseCut & \textbf{24.284} & \textbf{0.748} & 0.044 \\
\hline %\bottomrule
 \end{tabular}
}
 \end{center}
 \caption{Quantitative comparison and runtime of algorithms using dataset of microscopy images. The results are mean values from the test set, which corresponds to 1,000 images.}
 \label{table:result_micro}
 \end{table}

Table \ref{table:result_micro} includes the quantitative evaluation results of the comparison algorithms and DeepPhaseCut using 1,000 microscopy images at the test phase. The proposed DeepPhaseCut significantly outperformed the competing algorithms in terms of PSNR and SSIM index. 
Specifically, the proposed method exceeds HIO with a large margin in terms of PSNR.
By comparison, we found that the proposed DeepPhaseCut increases PSNR by $0.44 \sim 2.07$ dB compared to PhaseCut and prDeep. 
Furthermore,  DeepPhaseCut is computationally negligible compared to other algorithms, since the reconstruction is performed using the feed-forward networks consisting of the phase network $\Gc_\Theta$ and image refinement network $\Hc_\Psi$ at the test phase.
This clearly confirms that the proposed method offers improvement in performance as well as efficient computational complexity.

\section{Discussion}
\label{sec:discussion} 

\subsection{Optimal Configuration}

\begin{figure}[!hbt] 	
\center{ 
\includegraphics[width=9cm]{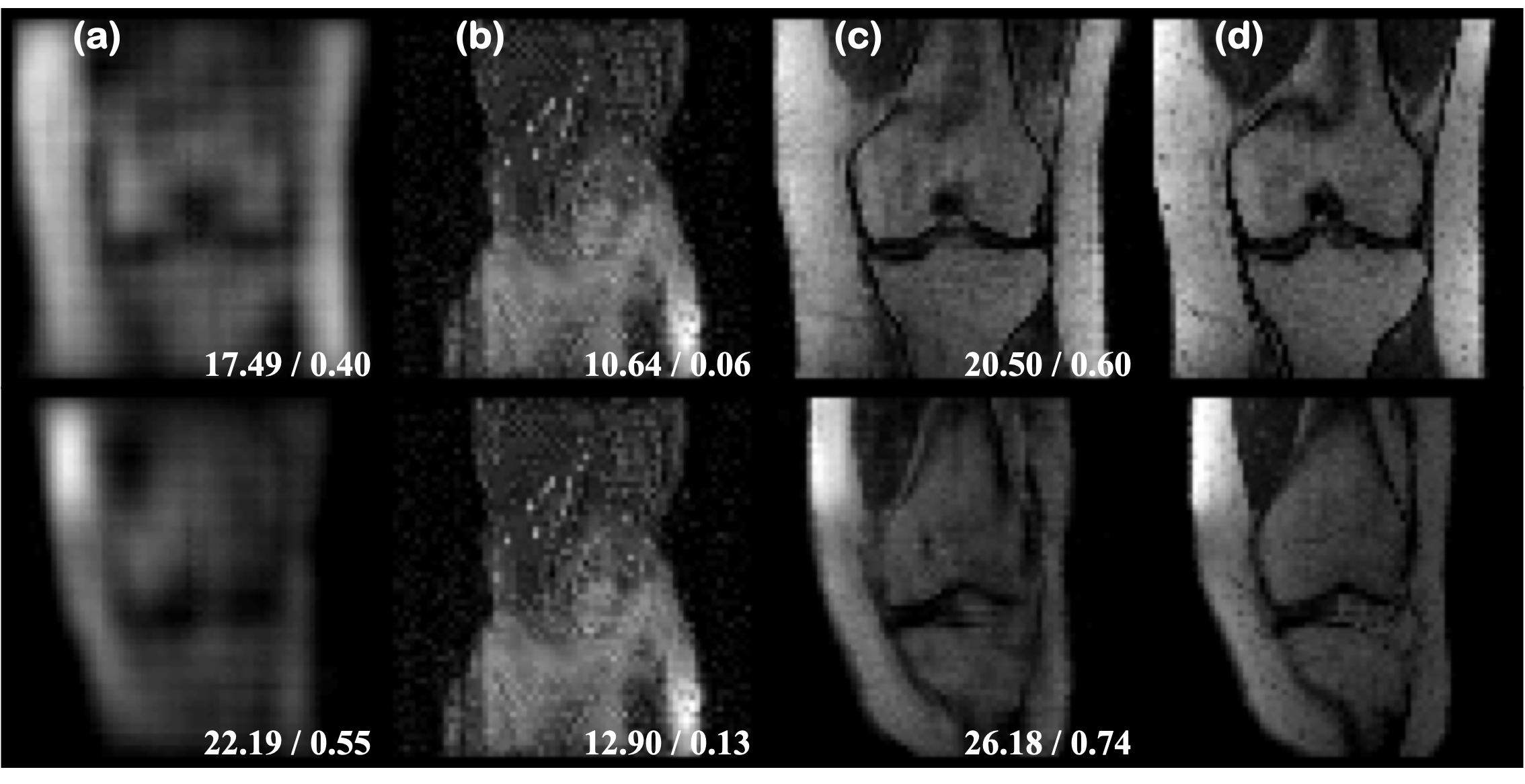}
\vspace*{-0.5cm}
}
\caption{Reconstruction results using an ablated network only consisting of (a) $\Gc_\Theta$ and (b) $\Hc_\Psi$ , respectively. (c) indicates reconstruction results using the proposed DeepPhaseCut. (d) is the target images. The values in the corners are PSNR / SSIM index for individual image with respect to the target image.}
\label{fig:abl_network}
\end{figure}

The proposed network is composed of two generators, $\Gc_\Theta$ and $\Hc_\Psi$. Ablation studies were performed using a single generator to restore the phase information by excluding $\Hc_\Psi$ or $\Gc_\Theta$ from the proposed DeepPhaseCut.
The reconstruction results using the fastMRI dataset are illustrated in Fig. \ref{fig:abl_network}.
The first column shows the results from the network only with $\Gc_\Theta$. 
The results from the network composed of only $\Hc_\Psi$ are demonstrated in the second column.
The reconstructed images using the proposed DeepPhaseCut are illustrated in the third column.
In Fig. \ref{fig:abl_network}(a), the results produced by the phase network $\Gc_\Theta$ are more blurry than the reconstruction using the proposed DeepPhaseCut and most of the details are not restored. 
In Fig. \ref{fig:abl_network}(b), the image refinement network $\Hc_\Psi$ trained alone provides erroneous reconstruction
results, and even resulted in the same reconstruction results from the different Fourier magnitude data. %magnitude in the frequency domain are different.
%This implies that then proper estimation of phase from the measurement $\bb$  is essential part of the Fourier phase retrieval.
By comparing Fig. \ref{fig:abl_network}(a) and Fig. \ref{fig:abl_network}(b), we concluded that the network $\Gc_\Theta$, which predicts the phase from the measurement $\bb$, is very important to restore the phase.

When both networks are used in our method,  we observed 3.01 $\sim$ 13.28 dB better than the ablated networks in terms of PSNR.
This confirmed that our DeepPhaseCut algorithm composed of image refinement after phase estimation is well suited for solving the phase retrieval problem.

\subsection{Ablation Study of Losses}

In order to verify the effect of each component in our designed loss function, we conducted an ablation study using microscopy data in which the building blocks of the loss function were individually excluded.
Table \ref{table:abl_loss} shows quantitative comparison of ablation study. 
We observe that the performance  is critically dependent upon the loss function.

Specifically, without the cycle consistency loss, 
the evaluation metric of PSNR and SSIM index become significantly degraded. This suggests the importance of the consistency in the image and the measurement domain.
%There is no constraint on training the image refinement network $\Hc_\Psi$ to keep the measured magnitude information $\bb$ when training without $\ell_{cycle}$, which results in critical performance degradation. 
On the other hand, the ablated network trained without $\ell_{pcut}$ produced better results than that without $\ell_{cycle}$, but there is still a performance drop compared to the proposed DeepPhaseCut.
This implies that minimizing the proposed $\ell_{pcut}$, derived from PhaseCut cost, plays an important role in estimating the phase from the given magnitude information. 
%Furthermore, a more suitable result, which is improved by the image refinement network $\Hc_\Psi$, is provided by the phase network $\Gc_\Theta$ thanks to $\ell_{pcut}$.
By comparing the performance between the model trained without $\ell_{adv}$ and the proposed DeepPhaseCut, we confirm that minimizing $\ell_{adv}$ encouraged the model to follow the actual distribution of the target domain $\Xc$ by addressing the ambiguity of the solutions.

%Furthermore, this implies that there is a limitation to address the ambiguity of the phase retrieval problems with just $\ell_{adv}$ or $\ell_{cycle}$.
This ablation study clearly confirmed the critical importance of each component of our proposed loss function in DeepPhaseCut.

\begin{table}[!hbt]
 \begin{center}
\resizebox{0.4\textwidth}{!}{
 \begin{tabular}{c|c|c}
  \hline
  & PSNR [dB] & SSIM index \\
\hline
\hline
Proposed w/o $\ell_{cycle}$ & 9.719 & 0.098 \\
Proposed w/o $\ell_{pcut}$ & 21.513 & 0.676 \\
Proposed w/o $\ell_{adv}$ & 23.653 & 0.741 \\
Proposed (DeepPhaseCut) & \textbf{24.284} & \textbf{0.748} \\
\hline
 \end{tabular}
}
 \end{center}
 \caption{Quantitative comparison of ablation study on loss function using the microscopy dataset. The results are mean values from the test set, which corresponds to 1,000 images.}
 \label{table:abl_loss}
 \end{table}

\section{Conclusion}
\label{sec:conclusion}
In this paper, we propose a novel unsupervised learning approach for phase retrieval inspired by the alternating projection method and the modern PhaseCut algorithm.
In particular, our DeepPhaseCut consists of two networks: one for the phase estimation and another for image refinement.
For the training of phase estimation network, we propose a novel loss, dubbed PhaseCut loss, which is originated from the analysis of the optimization problem in PhaseCut.
Furthermore, to address the ambiguity of the phase retrieval problem, we employed the adversarial loss and cycle consistency loss so that the architecture of DeepPhaseCut is similar to cycleGAN. 
Thanks to our novel architecture, the network of DeepPhaseCut can be trained without the need for the matched data as the label.
We performed extensive experiments using real-valued and complex-valued data to verify the versatility of the proposed DeepPhaseCut for phase retrieval. 
Our method provided superior reconstruction results, in which both the support and the details can be properly restored with a single feed-forward network, even though the network was trained using unmatched dataset. Thus, the proposed DeepPhaseCut may provide a new research direction for Fourier phase retrieval problem.

% if have a single appendix:
%\appendix[Proof of the Zonklar Equations]
% or
%\appendix  % for no appendix heading
% do not use \section anymore after \appendix, only \section*
% is possibly needed

% use appendices with more than one appendix
% then use \section to start each appendix
% you must declare a \section before using any
% \subsection or using \label (\appendices by itself
% starts a section numbered zero.)
%

\section*{Acknowledgment}
This work was supported by the National Research Foundation (NRF) of Korea grant NRF-2020R1A2B5B03001980.

%\appendices
%\section{Proof of the First Zonklar Equation}
%Appendix one text goes here.

\bibliographystyle{IEEEtran}
\bibliography{submit_bib}

\end{document}